\documentclass[12pt]{article}
\usepackage[utf8]{inputenc}
\usepackage{amsmath,amsthm,amssymb,fullpage,graphicx}
\usepackage{algorithm,algorithmic,verbatim,color}
\usepackage{mathtools}
\usepackage[round]{natbib}
\usepackage[colorlinks=true, citecolor=blue, linkcolor=black, urlcolor=magenta]{hyperref}
\usepackage{cleveref}
\usepackage{algorithm,algorithmic}
\usepackage{graphicx, caption} 
\usepackage{xcolor}
\usepackage{enumitem}
\usepackage[T1]{fontenc}
\usepackage{wrapfig}

\DeclareMathOperator{\Var}{Var}
\DeclareMathOperator{\tr}{tr}
\DeclareMathOperator{\diag}{diag}

\theoremstyle{plain} 
\newtheorem{theorem}{Theorem}[section]
\newtheorem{lemma}[theorem]{Lemma}

\newtheorem{corollary}[theorem]{Corollary}
\theoremstyle{definition} 

\theoremstyle{remark} 

\numberwithin{theorem}{section}

\title{Sampling and Loss Weights in Multi-Domain Training}
\author{
  \normalsize Mahdi Salmani$^{1}$ \quad
  Pratik Worah$^{2}$ \quad
  Meisam Razaviyayn$^{1}$ \quad
  Vahab Mirrokni$^{2}$ \\[2ex]
  \small $^{1}$University of Southern California \quad
  $^{2}$Google Research \\[1ex]
  {\small \texttt{\{salmanis, razaviya\}@usc.edu, \{pworah, mirrokni\}@google.com}}
}

\date{}%

\begin{document}
\maketitle
\begin{abstract}
    In the training of large deep neural networks, there is a need for vast amounts of training data. To meet this need, data is collected from multiple domains, such as Wikipedia and GitHub. These domains are heterogeneous in both data quality and the diversity of information they provide. 
    This raises the question of how much we should rely on each domain.
    Several methods have attempted to address this issue by assigning sampling weights to each data domain using heuristics or approximations. As a first step toward a deeper understanding of the role of data mixing, this work revisits the problem by studying two kinds of weights: \textit{sampling weights}, which control how much each domain contributes in a batch, and \textit{loss weights}, which scale the loss from each domain during training. Through a rigorous study of linear regression, we show that these two weights play complementary roles.
    First, they can reduce the variance of gradient estimates in iterative methods such as stochastic gradient descent (SGD). Second, they can improve generalization performance by reducing the generalization gap. We provide both theoretical and empirical support for these claims. We further study the joint dynamics of \textit{sampling weights} and \textit{loss weights}, examining how they can be combined to capture both contributions. 
\end{abstract}

\section{Introduction}

The success of modern large-scale models has been fueled by training on massive datasets that combine examples from many heterogeneous domains \citep{devlin2019bert, brown2020language, anil2023gemini}. These domains differ not only in size but also in reliability, noise level, and information. A common practice in large-model pretraining pipelines is to assign each domain a single scalar weight, either proportional to dataset size or tuned heuristically, and then train on the resulting mixture \citep{xie2023doremi, albalak2023efficient, fan2024doge, li2025pike, xie2025chameleon}. While simple and effective in practice, this \emph{single-weight} perspective implicitly assumes that all aspects of domain heterogeneity can be captured by a single parameter. 

Large-model training motivates us to take a closer look at the underlying nature of these weights. At its core, assigning a single domain weight conflates two fundamentally different roles: how much influence a domain should have on the learning objective, and how frequently it should be sampled during optimization. We argue that, even in the absence of explicit domain adaptation, at least two distinct forms of weighting naturally arise:
\begin{enumerate}
    \item \textbf{Loss weights}, which determine how much the empirical risk from each domain contributes to the optimization objective. These capture the reliability and generalization capabilities of domains: cleaner, less noisy sources should contribute more, while noisier ones should be downweighted.
    \item \textbf{Sampling weights}, which determine how often examples from each domain are drawn during stochastic optimization. Since gradient variance differs across domains, adjusting sampling frequencies can reduce stochastic noise and improve convergence. These weights therefore act on the stability and efficiency of the optimization process.
\end{enumerate}

By separating these two roles, we uncover a richer picture of domain weighting. 

Our contribution is to study two types of weighting schemes, propose practical estimators for them, and evaluate their impact through regression experiments. Specifically:
\begin{itemize}
    \item In linear regression, we show that loss weights can be derived from generalized least squares (GLS): domains with higher conditional label variance receive lower weights. We then propose an efficient single-pass estimator that avoids iterative re-estimation.
    \item We extend this idea to empirical risk minimization by introducing a dynamic update rule that adjusts loss weights during training based on observed errors.
    \item For sampling weights, we analyze them through the lens of variance reduction in stochastic optimization. We propose a strategy that allocates more samples to domains with higher gradient variance, improving the optimization part.
    \item We validate these approaches in experiments on linear and logistic regression, showing that loss and sampling weights provide distinct, complementary benefits, with each yielding measurable improvements on its own.
\end{itemize}

In summary, domain weighting is not one-dimensional but involves both loss and sampling weights. Recognizing this structure leads to clearer theory and practical improvements in estimation and optimization.

\paragraph{Related Work.}  

The study of weighting data points has a long history in statistics and econometrics. Early work on generalized least squares (GLS) showed how weighting could be used to address heteroskedasticity and yield efficient estimators \citep{aitken1935least}. This line of research developed into weighted least squares and heteroskedasticity-consistent methods, which remain central in modern econometrics \citep{wooldridge2010econometric, greene2018econometric}. 

A complementary perspective comes from influence function analysis. Introduced in robust statistics \citep{hampel1974influence}, influence functions quantify how small perturbations or reweightings of data points affect an estimator. This framework was later extended to regression diagnostics \citep{cook1982residuals} and has recently been adopted in machine learning to study the sensitivity of models to training examples \citep{koh2017understanding}. The influence function view emphasizes that weighting is not only a matter of efficiency but also of robustness and interpretability.

In machine learning, weighting has appeared in various forms of reweighting and importance sampling. These include classical importance-weighted empirical risk minimization and variance reduction techniques for stochastic optimization \citep{shimodaira2000improving, defazio2014saga}. 
Most directly related to our setting are domain mixture strategies for large-model pretraining. In practice, large-scale training pipelines often rely on simple heuristics such as proportional-to-size sampling or manually tuned mixture weights. Recent work has sought to make these mixtures more principled. DoReMi \citep{xie2023doremi} learns mixture weights through a teacher–student scheme, where a teacher trained on a uniform mixture guides reweighting by comparing per-domain losses. DoGE \citep{fan2024doge} learns sampling weights via bi-level optimization to favor domains that improve generalization. Pike \citep{li2025pike} introduces adaptive mixing strategies that account for gradient conflicts across tasks.
Similarly, large-scale multimodal models such as Gemini \citep{anil2023gemini} employ curated mixtures of datasets, though often without a principled justification for the weighting scheme. These approaches, however, generally treat domain weighting as a single scalar factor, mostly as sampling weights, without separating its impact on generalization from its impact on optimization.

Our work builds on these classical and modern perspectives but makes a distinct contribution: we highlight that in multi-domain learning, two different types of weights naturally arise, namely loss weights and sampling weights, and we develop algorithms for estimating both. This distinction provides a clearer conceptual framework for understanding weighting, while offering practical improvements in controlled experimental settings.

\section{Problem Setup}
In this section, we introduce three distinct notions of weight. The first type influences the model’s final test performance. The second type helps reduce the generalization gap. The third type contributes to faster convergence during optimization. We now examine each of these notions in detail.

\paragraph{Domain-weighted Population Risk}
Consider $K$ data domains with distributions $\mathcal{D}_1, \ldots, \allowbreak \mathcal{D}_K$, each supported on a common space $\mathcal{Z} \subset \mathbb{R}^{d}$. 
Let $\ell : \Theta \times \mathcal{Z} \to \mathbb{R}$ denote a loss function.
We define the domain-weighted population risk as
\begin{equation}
    \mathcal{L}_{\pi}(\theta) \;=\; \sum_{i=1}^{K} \pi_i \, \mathcal{L}_i(\theta),
\end{equation}
where $\mathcal{L}_i(\theta) = \mathbb{E}_{z \sim \mathcal{D}_i}\!\left[\ell(\theta, z)\right]$ denotes the population risk for domain $i$, and $\pi_i$ represents the weight assigned to that domain. These weights quantify the relative impact of the domains on overall model test performance. For instance, if $\pi_i \propto m_i$, where $m_i$ denotes the probability that a randomly sampled data point comes from domain $i$, then the objective recovers the standard population risk under the mixture distribution, which is optimal in the absence of any distribution shift between training and test data. When $\pi_i = 1$ for all $1 \leq i \leq K$, the objective reduces to \emph{universal generalization} \citep{fan2024doge}, in which all domains are treated as equally important. Alternatively, if the goal is to apply a minimax strategy and minimize the worst-case domain performance, one can employ \emph{Group Distributionally Robust Optimization} (Group DRO) \citep{sagawadistributionally,xie2023doremi}.
\paragraph{Domain-weighted Empirical Risk}
For a realized dataset $\mathcal{S} = \{\mathcal{S}_1, \ldots, \mathcal{S}_K\}$, where each $\mathcal{S}_i$ consists of i.i.d.\ samples from its corresponding distribution $\mathcal{D}_i$, we define the domain-weighted empirical risk as
\begin{equation}
    \hat{\mathcal{L}}_{\mathcal{S}, \pi, w}(\theta) \;=\; \sum_{i=1}^{K} \pi_i w_i \, \hat{\mathcal{L}}_{\mathcal{S}_i}(\theta),
\end{equation}
where $\hat{\mathcal{L}}_{\mathcal{S}_i}(\theta) = \tfrac{1}{|\mathcal{S}_i|} \sum_{z \in \mathcal{S}_i} \ell(\theta, z)$ denotes the empirical risk on domain $i$. As a special case, choosing $w_i \propto |\mathcal{S}_i| / \pi_i$ recovers the standard empirical risk over the pooled dataset. Another natural choice is $w_i = 1$, which yields an unbiased estimator of the corresponding domain-weighted population risk. Intuitively, the weights $w_i$ reflect how much we rely on the empirical risk from each domain. If $\hat{\mathcal{L}}_{S_i}(\theta)$ is relatively closer to its population risk $\mathcal{L}_i(\theta)$ compared to other domains (i.e., it generalizes better), then it should be assigned a larger weight than under uniform weighting. Conversely, if it is relatively less reliable, it should receive a smaller weight.
\paragraph{Domain-weighted Optimization Sampling}
The final notion concerns the sampling frequency, or weight, with which data from each domain is visited during optimization. Specifically, we aim to compute the domain-weighted ERM (empirical risk minimizer)
\begin{equation}
    \hat{\theta} \;=\; \arg \min_{\theta} \sum_{i=1}^{K} \pi_i w_i \, \hat{\mathcal{L}}_{\mathcal{S}_i}(\theta).
\end{equation}
This objective is typically solved using iterative optimization methods such as SGD or Adam. In this paper, we primarily focus on SGD, which updates the parameters according to
\begin{equation}
    \theta_{t+1} \gets \theta_{t} - \eta\,g_t,
\end{equation}
where $g_t$ is an unbiased estimator of 
$\hat{\mathcal{L}}_{\mathcal{S}, \pi, w}(\theta)$.
To obtain $g_t$, we draw a mini-batch $\mathcal{B}_t$. In the multi-domain setting, there are several strategies for constructing such batches. We focus on an effective approach in practice, namely the \emph{mixing strategy} \citep{devlin2019bert, team2023gemini, li2025pike}. In this approach, the mini-batch is formed as $\mathcal{B}_t = \{\mathcal{B}_t^1, \ldots, \mathcal{B}_t^K\}$, where each $\mathcal{B}_t^i$ consists of i.i.d.\ samples drawn uniformly at random from $\mathcal{S}_i$, i.e., $\mathcal{B}_t^i \sim \mathcal{S}_i$. The resulting gradient estimator is then
\begin{equation}
    g_t \;=\; \sum_{i=1}^{K} \frac{\pi_i w_i}{|\mathcal{B}_t^i|} \sum_{z \in \mathcal{B}_t^i} \nabla_\theta \ell(\theta, z).
\end{equation}
A natural question is how many samples to draw from each domain when constructing the batch. Intuitively, more samples should be drawn from domains whose corresponding gradients exhibit higher variance, as this reduces the overall variance of the estimator and leads to faster convergence.

\paragraph{Finding the Optimal Weights} 
There has been extensive work on selecting optimal weights for the domain-weighted population risk, especially in the domain adaptation literature \citep{shimodaira2000improving, farahani2021brief, xia2024less}. These works typically aim to correct distributional shifts by reweighting samples or domains so that the weighted population risk better reflects the target distribution. Motivated by this line of research, we turn our attention to the other two types of weights, assuming that the population mixture proportions $\pi_i$ are given. Our goal is to investigate how these weights can be chosen to improve both generalization and optimization performance.

\section{Weights for Empirical Risk}
In this section, we discuss the impact of domain weights on improving generalization and examine how such weights can be obtained. To this end, we begin by studying linear regression, which provides insight into the characteristics of these weights. We then show how this approach can be generalized to arbitrary models.

\subsection{Understanding the Linear Regression Case}
Assume a linear latent variable model in which the true parameter is shared across different data domains, while the label noise varies between domains. Formally, for each sample $z = (\mathbf{x}, y) \sim \mathcal{D}_i$, we have
\begin{equation}
\label{eq:linear_latent}
y = \theta_{\text{gt}}^\top \mathbf{x} + \epsilon,
\end{equation}
where $\theta_{\text{gt}}$ is shared across domains, $\mathbb{E}[\epsilon] = 0$, and $\operatorname{Var}(\epsilon) = \sigma_i^2$, with $\sigma_i^2$ representing the domain-specific label noise variance. 
To estimate $\theta_{\text{gt}}$ in this setting, one may employ the squared loss 
\(
\ell(\theta, z) = \bigl(\theta^\top \mathbf{x} - y \bigr)^2
\)
within the empirical risk minimization (ERM) framework, which yields the ordinary least squares (OLS) estimator
\begin{equation}
    \hat{\theta}_{\mathrm{OLS}} = (\mathbf{X}^\top \mathbf{X})^{-1} \mathbf{X}^\top \mathbf{y},
\end{equation}
where $\mathbf{X} = \bigl[\mathbf{x}_1 \mid \ldots \mid \mathbf{x}_n \bigr]^\top$ and 
$\mathbf{y} = \bigl[y_1 \mid \ldots \mid y_n \bigr]^\top$ for $(\mathbf{x}_i, y_i) \in \mathcal{S}$. 
This estimator, however, can be improved by assigning domain-specific weights, as guaranteed by the Aitken theorem~(\Cref{thm:aitken}).

\begin{theorem}[\cite{aitken1935least}]
\label{thm:aitken}
Consider the linear model $\mathbf{y} = \mathbf{X}\theta + \mathbf{\epsilon}$, where $\mathbb{E}[\mathbf{\epsilon}] = 0$ and $\operatorname{Var}(\mathbf{\epsilon}) = \mathbf{\Sigma}$, with $\mathbf{\Sigma}$ a positive definite matrix. The generalized least squares (GLS) estimator
\begin{equation}
    \hat{\theta}_{\mathrm{GLS}} = (\mathbf{X}^\top \mathbf{\Sigma}^{-1} \mathbf{X})^{-1} \mathbf{X}^\top \mathbf{\Sigma}^{-1} \mathbf{y}
\end{equation}
is the best linear unbiased estimator, achieving the minimum variance among linear unbiased estimators.
\end{theorem}
In our setting, the noise terms are uncorrelated, so $\mathbf{\Sigma}$ is diagonal. The optimal weights then follow directly from \Cref{thm:aitken}, yielding \Cref{corr:optimal_weights_erm_linear}.

\begin{corollary}
\label{corr:optimal_weights_erm_linear}
For the linear latent variable model in \Cref{eq:linear_latent}, the optimal weights $w_i^\star$ in domain-weighted empirical risk minimization are given by
\begin{equation}
    w_i^\star \propto \frac{1}{\sigma_i^2}.
\end{equation}
\end{corollary}
\Cref{corr:optimal_weights_erm_linear} aligns with our intuition. Domains that are relatively \emph{noisier} and generalize less should receive reduced weight, while less noisy domains should receive increased weight. 

So far, we have seen that in the linear regression setting, the optimal domain-weighted empirical risk can be computed when the noise variances for each domain are known. In practice, however, these variances are typically unknown, and the weights must be estimated. A standard method for this purpose is Feasible Generalized Least Squares (FGLS) \citep{judge1985theory, wooldridge2010econometric, greene2018econometric}.
FGLS begins by computing the OLS estimator $\hat{\theta}_{\text{OLS}}$. 
The residuals are then used to estimate the domain noise variances and, consequently, the corresponding domain weights, 
\begin{equation}
    \hat{\sigma}^2_i \propto \frac{1}{|\mathcal{S}_i|} \sum_{(\mathbf{x},y)\in \mathcal{S}_i} \bigl(\hat{\theta}_{\text{OLS}}^\top \mathbf{x} - y \bigr)^2, 
    \qquad 
    \hat{w}_i \propto \frac{1}{\hat{\sigma}_i^2}.
\end{equation}
There are two main problems with FGLS. First, the procedure requires training the model at least twice (and potentially multiple iterations to refine the estimates). Second, the validity of the estimation can be problematic. For instance, in an over-parameterized setting where $d > |\mathcal{S}|$, the residuals vanish, and the weight estimates become ill-defined. To overcome these issues, we propose \textbf{One-shot FGLS}. 

\subsubsection{One-shot FGLS}
As mentioned, waiting until after training the entire model to update the domain weights is not ideal. A natural solution is to use an iterative algorithm that estimates the weights during training and then applies these estimates. Concretely, we may draw a sample set from the data and estimate the noise variances from these samples.

At the same time, if the samples used for variance estimation are drawn from data already used to train the model, we may face the same issue as in FGLS, where the training data are fitted so closely that the loss on this set is no longer meaningful. In such cases, the distribution of training residuals can deviate significantly from the true distribution, for example the distribution of validation residuals. That said, there are training scenarios where this issue is less pronounced. For instance, in the training of language models, each example is typically seen only a few times due to the abundance of data, which mitigates the problem. 

We propose a method inspired by FGLS that estimates variances during training (\Cref{alg:one-shot-fgls}). To this end, we select a subset of data points to estimate the expected loss and then apply a smooth update rule to adjust the weights (Line 16, \Cref{alg:one-shot-fgls}). It is important that this subset act as a validation set, meaning it must be independent of the model parameters. One way to ensure this is to split the training data into two parts using a ratio $\rho$, and use the smaller part for estimation. We then show that this method approaches the performance of the optimal estimator as the number of data points grows (\Cref{thm:asymptotic-norm}).

\begin{theorem}[Informal]
\label{thm:asymptotic-norm}
As the sample size increases, the mean squared error of the estimator produced by \Cref{alg:one-shot-fgls} decays at the same asymptotic rate as that of the optimal estimator; in particular, the ratio of their mean squared errors converges to $1$.
\end{theorem}

\begin{algorithm}
\caption{One-shot FGLS}
\begin{algorithmic}[1]
\label{alg:one-shot-fgls}
\REQUIRE Iterations $T$, update interval $T_0$, batch size $B$, initial fraction $\rho$, learning rate $\eta$
\STATE Initialize $\theta_0 \gets \mathbf{0}$, \quad $M \gets (1-\rho)\tfrac{T_0}{T}$
    \STATE Sample $S^{\text{train}}_i \subseteq S_i$ with $|S^{\text{train}}_i| = \rho |S_i|$ for $i \in [k]$
\FOR{$t = 0$ \TO $T-1$}
        \STATE Sample batch $B_i \subseteq S^{\text{train}}_i$ for $i \in [k]$
    \STATE $g_t \gets \dfrac{1}{\sum_{i=1}^K w_i^{(t)} \pi_i} 
                 \sum_{i=1}^K \frac{w_i^{(t)} \pi_i}{|B_i|}\sum_{z \in B_i} \nabla_\theta \ell(\theta_t, z)$
    \STATE $\theta_{t+1} \gets \theta_t - \eta g_t$
    \FOR{$i = 1$ \TO $K$}
        \IF{$t \bmod T_0 = T_0 - 1$}
            \STATE $\mathcal{R}_i \gets S_i \setminus S^{\text{train}}_i$
            \STATE Sample $B'_i \subseteq \mathcal{R}_i$ with $|B'_i| = M|S_i|$
            \STATE $w_i^{(t+1)} \gets (1 - \gamma)w_i^{(t)} + \dfrac{\gamma}{\tfrac{1}{|B'_i|}\sum_{z \in B'_i}\ell(\theta_{t+1}, z)}$
    \ELSE 
        \STATE $w_i^{(t+1)} \gets w_i^{(t)}$
    \ENDIF
    \ENDFOR
\ENDFOR
\end{algorithmic}
\end{algorithm}

\subsection{Beyond Linear Regression}
The next step is to extend the proposed method to a general learning problem. 
Unlike linear regression, however, obtaining a direct counterpart to 
\Cref{thm:aitken} for the general case that characterizes the behavior of the 
optimal ERM weights is not feasible. Instead, we focus on deriving a general 
upper bound on generalization with respect to the weights, and then optimize 
the weights to minimize this bound. One approach is to use variance-based 
generalization bounds, as stated in \Cref{thm:generalization-bound}.

\begin{theorem}[Informal]
    \label{thm:generalization-bound}
    Assume the loss is bounded for each domain. For a sufficiently large 
    validation set $\mathcal{V} = \{\mathcal{V}_1, \ldots, \mathcal{V}_K\}$, 
    the following inequality holds with high probability for some constant 
    $C$ and for all $\theta$:
    \begin{align}
       \bigl(\mathcal{L}_\pi(\theta) 
       - \hat{\mathcal{L}}_{\mathcal{V}, \pi, w}(\theta)\bigr)^2 
       &\leq 2 \left(\sum_{i=1}^K \pi_i (1-w_i)\, 
       \mathcal{L}_i(\theta)\right)^2 \nonumber \\
       &\quad + C \sum_{i=1}^K 
       \frac{\pi_i^2 w_i^2}{|\mathcal{V}_i|}\, 
       \Var_i(\theta),
    \end{align}
    where $\Var_i(\theta) = 
    \Var_{z \sim \mathcal{D}_i}\!\bigl(\ell(\theta, z)\bigr)$.
\end{theorem}

The main goal is to reduce the bound in \Cref{thm:generalization-bound}. 
In particular, we aim to estimate the optimal weights and update them smoothly 
towards this value. To this end, we minimize the upper bound and apply a single 
step of mirror descent to update the parameters. Assuming $|\mathcal{V}_i| \propto \pi_i$, 
we obtain the following update rule:
\begin{equation}
    w_i^{(t+1)} \propto w_i^{(t)} \exp\left(
      \gamma_1\, \pi_i G(t)\,\mathcal{L}_{i}(\theta_t) 
      - \gamma_2\, \pi_i w_i^{(t)}\,
        \Var_{i}(\theta_t)
    \right),
\end{equation}
where 
\(
    G(t) = \left(\sum_{j=1}^K \pi_j \bigl(1 - w_j^{(t)}\bigr)\, \mathcal{L}_j(\theta_t)\right),
\)
and $\gamma_1, \gamma_2$ are tunable hyperparameters. We adopt the same idea as in One-shot FGLS to estimate the variance and expected loss for each domain using a temporary holdout dataset, and then update the weights accordingly. 
We refer to this update rule as \textbf{ERMA} weighting (\textbf{ERM} \textbf{A}ware Weighting).

One useful feature of this update is that estimating the mean and variance of domain losses is not computationally demanding, which is encouraging for practical use. However, tuning the associated parameters can still be challenging. Moreover, in large-scale pretraining, where data are typically seen only once, the same samples can be used for both training and estimation.

Another notable aspect of this formulation is that $G(t)$ can shed light on the role of low-loss, medium-loss, and high-loss data points in the training process. In particular, there has been extensive work on the effect of pruning data based on their loss values \citep{marion2023less,sow2025dynamic}. However, no general rule has emerged: in some cases, removing high-loss examples improves model performance, while in others it has the opposite effect. Our formulation offers one possible explanation, since $G(t)$ can take both positive and negative values.

\section{Weights for Gradient Estimation}
Gradient estimation is central to stochastic optimization. As shown in \Cref{tab:sgd_convergence_short_nc}, the variance of the gradient estimator directly affects the convergence rate. This variance can differ across domains; intuitively, domains with greater data redundancy tend to exhibit lower gradient variance because their samples are more similar to one another. 

\begin{table}[t]
\centering
\caption{
Convergence rates of SGD under different regimes. 
SC denotes strongly convex, and all results assume $L$-smoothness. 
Here $R = \|\theta_0 - \theta^\star\|$, 
$\sigma^2 = \mathbb{E}\bigl[\|\nabla \ell(\theta, z) - \nabla \hat{\mathcal{L}}(\theta)\|^2\bigr]$ 
is an upper bound on the variance of the stochastic gradients, 
and $\Delta = \hat{\mathcal{L}}(\theta_0) - \hat{\mathcal{L}}^\star$ is the initial suboptimality. 
As can be seen, reducing $\sigma$ improves the convergence rate.
}
\label{tab:sgd_convergence_short_nc}
\renewcommand{\arraystretch}{1.1}
\setlength{\tabcolsep}{3pt}
\small
\begin{tabular}{|c|c|c|}
\hline
\textbf{Setting} & \textbf{Step size} & \textbf{Rate} \\
\hline
Convex & $\eta_t \sim 1/\sqrt{t}$ &
$\mathcal{O}\!\left(\tfrac{L R^2}{T} + \tfrac{\sigma R}{\sqrt{T}}\right)$ \\
\hline
$\mu$-SC & $\eta_t \sim 1/(\mu t)$ &
$\tilde{\mathcal{O}}\!\left(\tfrac{\sigma^2}{\mu T}\right)$ \\
\hline
$\mu$-SC & $\eta = \Theta(1/L)$ &
$\mathcal{O}\!\left((1-\mu/L)^T\right) + \mathcal{O}\!\left(\tfrac{\sigma^2}{\mu L}\right)$ \\
\hline
Non-convex & $\eta_t \sim 1/\sqrt{t}$ &
$\tilde{\mathcal{O}}\!\left(\tfrac{L\Delta}{\sqrt{T}} + \tfrac{\sigma^2}{\sqrt{T}}\right)$ \\
\hline
\end{tabular}
\end{table}

This highlights the importance of domain-specific sampling strategies in order 
to reduce the total variance. Since our approach relies on mixed-domain sampling, 
at each iteration we solve the following optimization problem to minimize the 
variance of the gradient estimate:
\begin{equation}
\begin{aligned}
(b_1^\star,\ldots,b_K^\star) 
&= \arg\min_{\mathbf b}\; 
    \mathbb{E}\Bigl[\,
        \bigl\|g_t - \nabla_{\theta}\hat{\mathcal{L}}_{\mathcal{S}}(\theta_t)\bigr\|^2
    \,\Bigr] \\
\text{s.t. }\;& b_i = \lvert \mathcal{B}_t^i\rvert \;\;\; \forall i \in \{1,\ldots, K\}, 
\qquad \sum_{i=1}^K b_i = B,
\end{aligned}
\end{equation}
where $B$ denotes the total batch size and $\mathcal{B}_t^i$ is the subset of 
samples drawn from domain $i$ at iteration $t$. 

Under the i.i.d.~sampling assumption, the variance decomposes as
\begin{equation}
    \mathbb{E}\Bigl[\,
        \bigl\|g_t - \nabla_{\theta}\hat{\mathcal{L}}_{\mathcal{S}}(\theta_t)\bigr\|^2
    \,\Bigr] 
    = \sum_{i=1}^K \frac{\pi_i^2 w_i^2}{b_i}\, v_i^2,
\end{equation}
with
\begin{equation}
\label{eq:vi-def}
    v_i^2 = \mathbb{E}_{z\sim \mathcal{S}_i}
        \Bigl\|\,\nabla \ell(\theta_t, z) 
        - \nabla_{\theta}\hat{\mathcal{L}}_{\mathcal{S}_i}(\theta_t)\Bigr\|^2
\end{equation}

Applying the method of Lagrange multipliers yields the optimal allocation
\begin{equation}
    -\frac{\pi_i^2 w_i^2 v_i^2}{b_i^2} + \lambda = 0
    \quad \implies \quad 
    b_i \;\propto\; \pi_i w_i v_i.
\end{equation}
This aligns with intuition. If the gradients are similar, the data points within a domain are less informative, so fewer samples are needed from that domain. 

Now that we know the optimal sampling strategy depends on the values of $v_i$, the question is how to estimate them. A straightforward approach would be to use a large batch of data at each step, but this is infeasible as it requires a substantial amount of data at every iteration. Instead, we estimate $v_i$ periodically, for example once every $T_1$ steps. While this provides a practical solution, there remains room for improving these estimation methods, which we leave for future work. 

\Cref{alg:va-sampling} provides an overview of SGD with this sampling method, which we refer to as \textbf{VA} (\textbf{V}ariance \text{A}ware) sampling. The algorithm shown is written for fixed $w_i$, but loss-based reweighting can be easily combined with sampling-based reweighting. We empirically study the effect of using both in the next section.

\begin{algorithm}
\caption{SGD with Variance Aware Sampling}
\begin{algorithmic}[1]
\label{alg:va-sampling}
\REQUIRE Iterations $T$, update interval $T_1$, batch size $B$, learning rate $\eta$, estimation batch size $B_e$
\STATE Initialize $\theta_0 \gets \mathbf{0}$
\FOR{$t = 0$ \TO $T-1$}
        \STATE Sample batch $B_i \subseteq S_i$ with $|B_i| = b_i^{(t)}\cdot B$ for $i \in [k]$
    \STATE $g_t \gets \dfrac{1}{\sum_{i=1}^K w_i \pi_i} 
                 \sum_{i=1}^K \frac{w_i \pi_i}{|B_i|}\sum_{z \in B_i} \nabla_\theta \ell(\theta_t, z)$
    \STATE $\theta_{t+1} \gets \theta_t - \eta g_t$
    \FOR{$i = 1$ \TO $K$}
        \IF{$t \bmod T_1 = T_1 - 1$}
        \STATE Sample batch $B'_i \subseteq S_i$ with $|B'_i| = B_e$
        \STATE Calculate $\hat{v}_i^{(t)}$ the estimate for $v_i$ (\Cref{eq:vi-def})
        \STATE $b_i^{(t+1)} \gets \pi_i w_i \hat{v}_i^{(t)}$
        \ELSE
        \STATE $b_i^{(t+1)} \gets b_i^{(t)}$
    \ENDIF
    \ENDFOR
    \STATE Normalize $b_i^{(t+1)}$
\ENDFOR
\end{algorithmic}
\end{algorithm}

\section{Experiments}
In this section, we empirically investigate the effects of applying loss weights and sampling weights, both individually and in combination. Our goal is to understand how each type of weight contributes to estimation quality and optimization dynamics when domains differ in reliability and variance. 

To this end, we consider two simple but illustrative models: linear regression and logistic regression. Despite their simplicity, these settings provide a controlled environment for analyzing weighting mechanisms without the additional complexity of large-scale architectures. Linear regression offers a direct connection to classical results such as FGLS, while logistic regression allows us to examine the behavior of weights in a non-linear classification setting. 
By comparing results across these experiments, we show that loss weights and sampling weights play complementary roles in improving estimation and optimization. 

Finally, we also examine the effect of using the weights in a setup with a neural network that has a single hidden layer, trained on a modified version of the MNIST dataset \citep{lecun1998mnist}.

\subsection{Linear Regression}
\begin{figure*}[!t]
    \centering
    \includegraphics[width=1\linewidth]{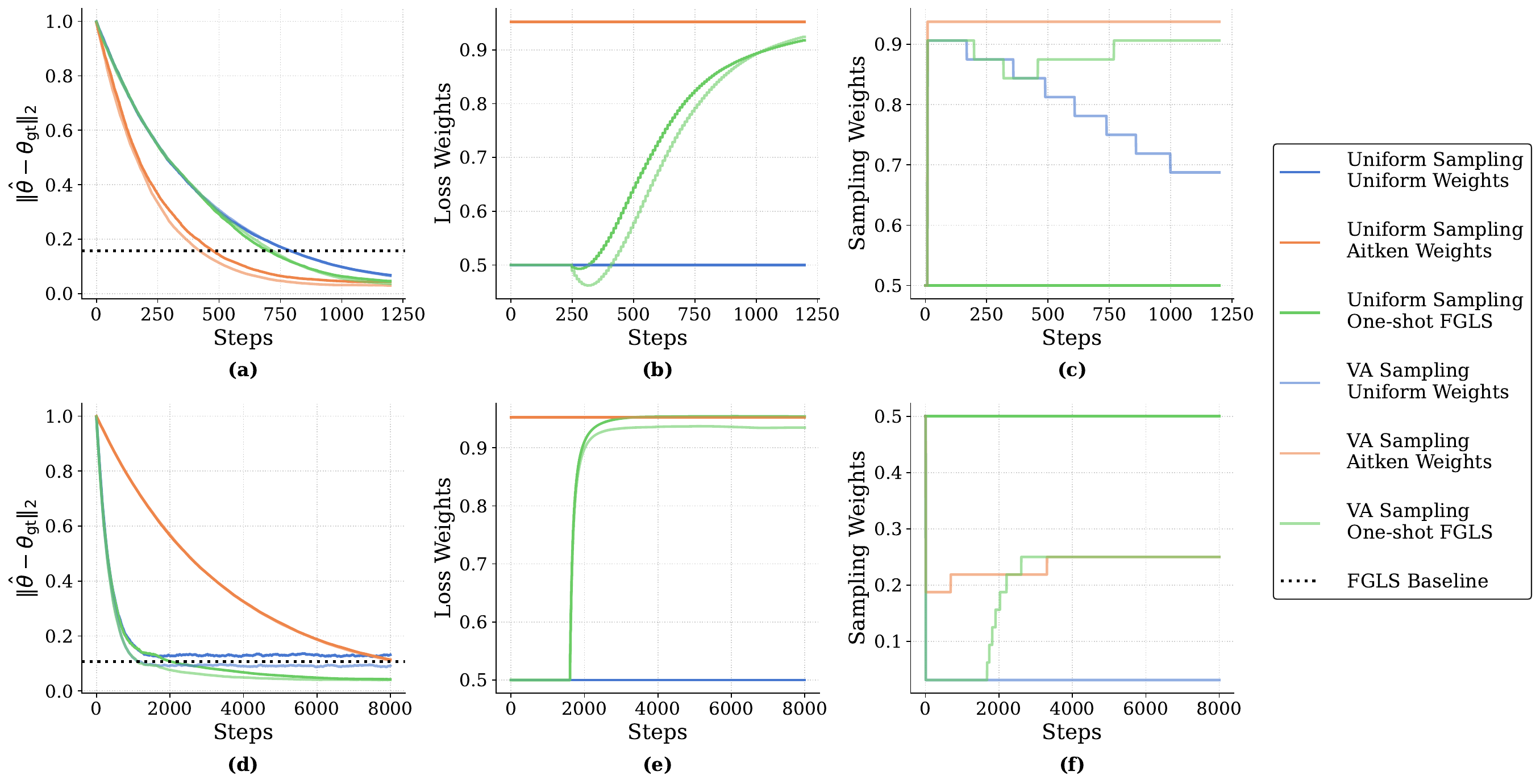}
    \caption{Performance of different methods in the linear regression example. 
        Figures \textbf{a to c} correspond to $(C_1, C_2) = (100, 1)$, while Figures \textbf{d to f} correspond to $(C_1, C_2) = (1, 100)$. 
        \textbf{a, d:} Distance between the estimated parameter and the ground-truth $\theta_{\mathrm{gt}}$ for each method. 
        \textbf{b, e:} Evolution of loss weights for domain one during training. 
        \textbf{c, f:} Evolution of sampling weights for domain one during training.
    }
    \label{fig:linear-regression}
\end{figure*}
\paragraph{Setup}
In the linear regression setting, we consider two data domains, $\mathcal{D}_1$ and $\mathcal{D}_2$. 
Samples in domain $i$ are generated as
\[
x \sim \mathcal{N}(0, C_i \mathbb{I}), 
\qquad 
y = \theta_{\mathrm{gt}}^\top x + \epsilon, 
\qquad 
\epsilon \sim \mathcal{N}(0, \sigma_i^2).
\]
We fix the data dimension to $d = 1000$, and set $\theta_{\mathrm{gt}}$ to be the normalized all-ones vector. We also assume $\pi_1 = \pi_2 = 0.5$
The noise variances are $\sigma_1^2 = 1$ and $\sigma_2^2 = 20$. 
For the scale parameters $C_i$, we consider two configurations: $(C_1, C_2) = (100, 1)$ and $(C_1, C_2) = (1, 100)$. 
This choice allows us to study the interaction between the loss and the sampling weights. 
Intuitively, increasing $C_i$ increases the gradient variance for domain $i$. 
By varying these values, we aim to investigate how the weights behave under different variance conditions. 
(For further discussion, see Appendix.)

We compare six training methods: (i) vanilla SGD, (ii) SGD with variance-aware (VA) sampling, (iii) SGD with optimal ERM weights from Theorem~\ref{thm:aitken}, (iv) SGD with optimal ERM weights and VA sampling, (v) One-shot FGLS (Algorithm~\ref{alg:one-shot-fgls}), and (vi) One-shot FGLS with VA sampling. 
All models are trained with a learning rate of $5 \times 10^{-5}$. 
For One-shot FGLS, we set $\gamma = 1$. 
For weight estimation, instead of splitting the initial dataset with ratio $\rho$ and then adding sampled data to the training set, we use a small subset of approximately 100 training points for all estimations. 
This choice simplifies the procedure and avoids additional complexity. 
Since early updates tend to produce poor and noisy estimations, we start updating the weights only after one-fifth of the total training steps.

\paragraph{Results} 
The results are presented in \Cref{fig:linear-regression}. 
Overall, both VA and One-shot FGLS prove effective, and we even observe additional improvements when combining them in the case $(C_1, C_2) = (1, 100)$. 

In the top row, corresponding to $(C_1, C_2) = (100, 1)$, both VA and OneShot FGLS assign higher sampling probabilities and larger loss weights to domain one. 
This aligns with intuition: domain one is more reliable due to lower label noise and more informative since its data points lie farther from the origin compared to domain two. 
Consequently, both loss and sampling weights emphasize domain one. 
Moreover, in this setting, One-shot FGLS converges to the optimal weights given by \Cref{thm:aitken}.

Turning to the second configuration, $(C_1, C_2) = (1, 100)$, we see a different behavior: VA tends to sample more from domain two, while One-shot FGLS upweights samples from domain one. 
A notable observation here is the suboptimal performance of the Aitken weights. 
We attribute this to the choice of learning rate, as training appears far from convergence under this setting.
\subsection{Logistic Regression}
\begin{figure*}[!t]
    \centering
    \includegraphics[width=1\linewidth]{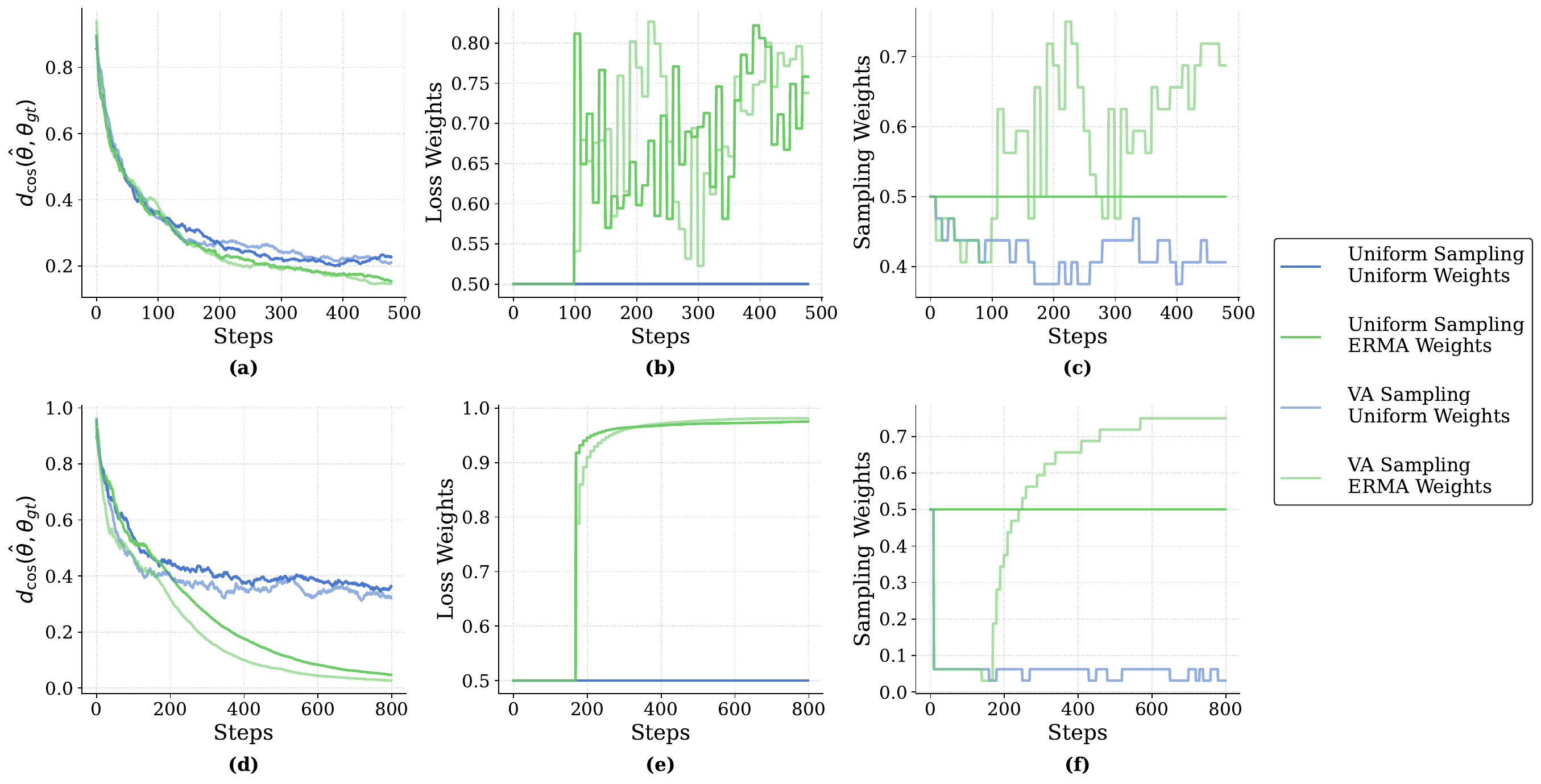}
    \caption{Performance of different methods in the logistic regression example. 
        Figures \textbf{a to c} correspond to $(C_1, C_2) = (100, 100)$, while Figures \textbf{d to f} correspond to $(C_1, C_2) = (10, 100)$. 
        \textbf{a, d:} Cosine distance between the estimated parameter and the ground-truth $\theta_{\mathrm{gt}}$ for each method. 
        \textbf{b, e:} Evolution of loss weights for domain one during training. 
        \textbf{c, f:} Evolution of sampling weights for domain one during training.}
    \label{fig:logistic-regression}
\end{figure*}
\paragraph{Setup}
For the logistic regression experiments, we again consider a two-domain setup. 
Samples in domain $i$ are generated as
\[
x \sim \mathcal{N}(0, C_i \mathbb{I}), 
\qquad 
y \sim \mathrm{Bernoulli}\!\left(\sigma(\theta_{\mathrm{gt}}^\top x)\right),
\]
where $\sigma(\cdot)$ denotes the sigmoid function. 
To incorporate label noise, we flip the generated label with probability $p_i$, i.e.,
\[
\tilde{y} = 
\begin{cases}
y, & \text{with probability } 1 - p_i, \\
1-y, & \text{with probability } p_i,
\end{cases}
\]
where $p_i$ is the flipping probability for domain $i$. 
Similar to the linear regression case, the data dimension is fixed at $1000$, $\theta_{\mathrm{gt}}$ is the normalized all-ones vector, and $\pi_1 = \pi_2 = 0.5$. 
We set $p_1 = 0$ and $p_2 = 0.2$. 
We again investigate two setups for the scale factors: $(C_1, C_2) = (100, 100)$ and $(C_1, C_2) = (10, 100)$. 
We evaluate four methods: (i) vanilla SGD, (ii) SGD with VA sampling, (iii) SGD with ERMA loss reweighting, and (iv) SGD with both VA and ERMA. 
For all methods, we use a learning rate of $10^{-4}$. 
For ERMA, we set $\gamma_1 = 0.01$ and $\gamma_2 = 0.05$. 
Similar to One-shot FGLS, ERMA includes a warm-up stage before estimating the weights.
For evaluation, we use the cosine distance
\begin{equation}
    d_{\cos}(a, b) = 1 - \frac{a^\top b}{\|a\|\|b\|}.
\end{equation}
Further discussion of these experiments is provided in Appendix.

\paragraph{Results}
As shown in \Cref{fig:logistic-regression}, both VA and ERMA improve classifier performance in both setups, with gains reflecting their complementary effects on stability and accuracy.

Starting with $(C_1, C_2) = (100, 100)$, we observe that ERMA places more emphasis on the less noisy domain, i.e., domain one, which is consistent with intuition. 
Interestingly, VA samples more from domain two when uniform weights are used, whereas it shifts to sampling more from domain one when combined with ERMA. 

In the second setup, ERMA outperforms uniform weighting by a large margin. 
Here, ERMA assigns even more weight to the less noisy domain compared to the previous case. 
This can be attributed to the fact that, in this setup, data points from the second domain are both noisier and located farther from the decision boundary, making them less useful for learning the boundary effectively.
\subsection{Neural Net}
\paragraph{Setup} To evaluate different methods, we use the MNIST dataset. Since this dataset does not have a natural notion of domains, we randomly split it into two groups and then randomly flip the labels of one group with a probability of $0.2$. We refer to this group as the noisy group. The same procedure is applied to the test split.

For the model, we use a simple neural network with a single hidden layer of 100 units and ReLU activations. To mimic the training dynamics of large language models, we set the total number of training steps such that each domain is visited at most once. Specifically, we set the total number of steps to $500$, which is sufficiently low to satisfy this condition. This approach allows us to avoid using a separate validation set. In other words, we use each data point before training on it to obtain the required terms for the methods. We also do not use any warm-up steps in these experiments.
\begin{figure*}[!t]
    \centering
    \includegraphics[width=1\linewidth]{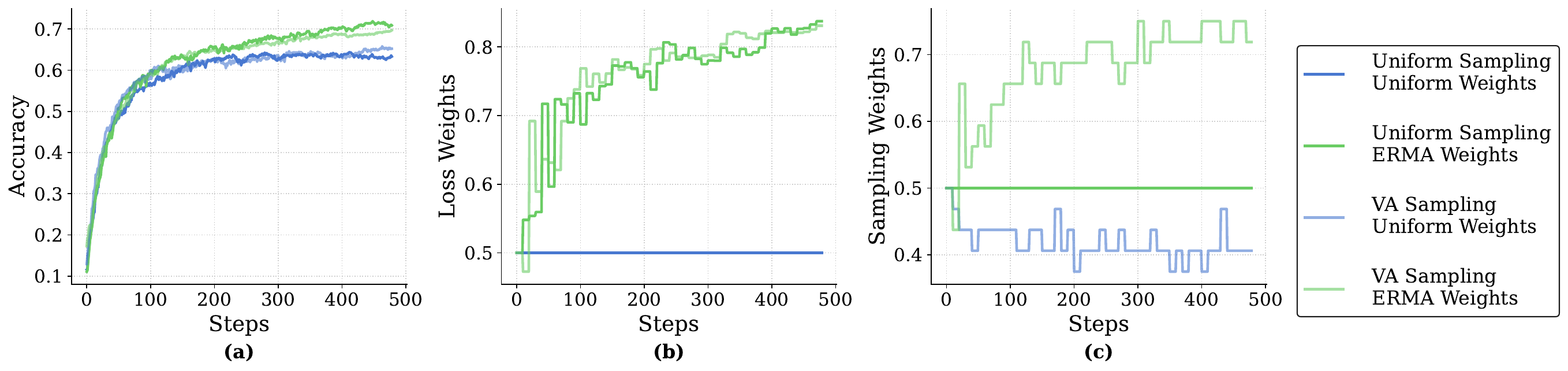}
    \caption{Performance of different methods in the neural net example.}
    \label{fig:mnist}
\end{figure*}

\paragraph{Results} \Cref{fig:mnist} illustrates the performance of each method on the MNIST dataset. In this setup, ERMA achieves the best results, improving upon uniform loss weighting, while VA appears to be ineffective. We attribute this to the high similarity of data inputs in both the clean and noisy groups, which makes the difference in gradient variance insignificant. For instance, without ERMA, VA assigns sampling weights of approximately $0.4$ and $0.6$ to the clean and noisy domains, respectively.
\section{Future Work}
Applying the discussed notion of weights in practice should be the next step. Deduplication is an important data processing step that can improve the performance of trained models. However, performing this process manually by identifying and removing similar data points can be challenging, particularly because defining an appropriate notion of similarity between examples is nontrivial. In this context, we can leverage VA to sample less frequently from domains that are repetitive or duplicative in the training dynamics, since VA naturally assigns lower sampling weights to such domains during training. 
Another interesting direction is to determine the optimal ERM weights that allow us to rely less on noisier domains. However, this choice depends on the structure of the data. If the data points are independent, we can directly apply ERMA. In contrast, this assumption does not hold in the training of autoregressive language models, where samples are inherently dependent. Extending ERMA to handle such cases would therefore be an important avenue for future work.
\section{Conclusion}
We studied the problem of domain weighting in multi-domain learning and showed that the common single-weight approach overlooks two distinct roles: \emph{loss weights}, which control domain contributions to the ERM objective, and \emph{sampling weights}, which regulate variance in stochastic optimization. 
To capture these effects, we proposed algorithms tailored to each: One-shot FGLS for estimating loss weights in linear regression, the ERMA update for adapting them in more general models, and the VA scheme for variance-aware sampling. 
Through experiments on linear and logistic regression, we observed that loss and sampling weights each provide measurable benefits, while their combination yields complementary improvements. 
These findings suggest that domain weighting is inherently two-dimensional rather than one-dimensional. 
This perspective not only provides a clearer theoretical framework for understanding weighting, but also points to promising future directions, such as adaptive procedures that jointly optimize both forms of weights in large-scale training and pretraining pipelines.
\clearpage
\bibliographystyle{plainnat}
\bibliography{arxiv}
\clearpage
\appendix
\thispagestyle{empty}

\onecolumn
\section{Proofs}
\subsection{Proofs of \Cref{thm:asymptotic-norm}}
Before analyzing the algorithm’s asymptotic behavior, we first present a lemma that provides an upper bound on an estimator’s relative suboptimality in terms of its loss weights and the noise variance.

\begin{lemma}
\label{lem:estimator_var_bound}
Let $\mathbf{X} \in \mathbb{R}^{n \times d}$ be the design matrix with rows $\mathbf{x}_1^\top,\ldots,\mathbf{x}_n^\top$, 
and let $\mathbf{y}=(y_1,\ldots,y_n)^\top \in \mathbb{R}^n$ be the labels with independent noise variances $\Var(y_i \mid \mathbf{x}_i)=\sigma_i^2$. 

Let the weighted estimator with optimal weights $\mathbf{W}^\star=\diag(\sigma_1^{-2},\ldots,\sigma_n^{-2})$ be
\[
\hat{\boldsymbol{\theta}}^\star = (\mathbf{X}^\top \mathbf{W}^\star \mathbf{X})^{-1}\mathbf{X}^\top \mathbf{W}^\star \mathbf{y}.
\]
For any WLS estimator with weights $\mathbf{W}=\diag(w_1,\ldots,w_n)$,
\[
\hat{\boldsymbol{\theta}}_{w} = (\mathbf{X}^\top \mathbf{W} \mathbf{X})^{-1}\mathbf{X}^\top \mathbf{W} \mathbf{y},
\]
the variance satisfies
\[
\frac{\tr\!\left(\Var(\hat{\boldsymbol{\theta}}_w)\right)}{\tr\!\left(\Var(\hat{\boldsymbol{\theta}}^\star)\right)}
\;\leq\;
\frac{\max_i w_i\sigma_i^2}{\min_j w_j\sigma_j^2}.
\]
\end{lemma}

\begin{proof}
Write $\mathbf{W} = \boldsymbol{\Lambda}\,\mathbf{W}^\star$ with $\boldsymbol{\Lambda} = \diag(w_1\sigma_1^2,\ldots,w_n\sigma_n^2)$,  
and let $\boldsymbol{\Omega} = \diag(\sigma_1^2,\ldots,\sigma_n^2) = (\mathbf{W}^\star)^{-1}$.  

Using $\Var(\mathbf{y}\mid \mathbf{X})=\boldsymbol{\Omega}$,
\begin{align*}
\tr\!\left(\Var(\hat{\boldsymbol{\theta}}_w)\right) 
&= \tr\!\left((\mathbf{X}^\top \boldsymbol{\Lambda} \mathbf{W}^\star \mathbf{X})^{-1} 
   \mathbf{X}^\top \boldsymbol{\Lambda} \mathbf{W}^\star \boldsymbol{\Omega}\, \boldsymbol{\Lambda} \mathbf{W}^\star \mathbf{X} 
   (\mathbf{X}^\top \boldsymbol{\Lambda} \mathbf{W}^\star \mathbf{X})^{-1}\right) \\
&= \tr\!\left((\mathbf{X}^\top \boldsymbol{\Lambda} \mathbf{W}^\star \mathbf{X})^{-1} 
   \mathbf{X}^\top \boldsymbol{\Lambda}^2 \mathbf{W}^\star \mathbf{X} 
   (\mathbf{X}^\top \boldsymbol{\Lambda} \mathbf{W}^\star \mathbf{X})^{-1}\right).
\end{align*}
Since $\boldsymbol{\Lambda}\succeq 0$ is diagonal,
\[
\boldsymbol{\Lambda}^2 \;\preceq\; \|\boldsymbol{\Lambda}\|_2\,\boldsymbol{\Lambda},
\qquad \|\boldsymbol{\Lambda}\|_2=\max_i (w_i\sigma_i^2).
\]
With $\mathbf{W}^\star\succeq 0$ diagonal, 
\[
\big(\|\boldsymbol{\Lambda}\|_2\,\boldsymbol{\Lambda}-\boldsymbol{\Lambda}^2\big)\mathbf{W}^\star \;\succeq\; 0,
\]
and by congruence with $\mathbf{X}$,
\[
\mathbf{X}^\top \boldsymbol{\Lambda}^2 \mathbf{W}^\star \mathbf{X} 
\;\preceq\; 
\|\boldsymbol{\Lambda}\|_2\, \mathbf{X}^\top \boldsymbol{\Lambda} \mathbf{W}^\star \mathbf{X}.
\]

Therefore,
\[
\tr\!\left(\Var(\hat{\boldsymbol{\theta}}_w)\right)
\;\le\; \|\boldsymbol{\Lambda}\|_2\; \tr\!\left((\mathbf{X}^\top \boldsymbol{\Lambda} \mathbf{W}^\star \mathbf{X})^{-1}\right)
= \big(\max_i w_i\sigma_i^2\big)\; \tr\!\left((\mathbf{X}^\top \boldsymbol{\Lambda} \mathbf{W}^\star \mathbf{X})^{-1}\right).
\]
Finally, by PSD ordering,
\[
(\min_j w_j\sigma_j^2)\,\mathbf{X}^\top \mathbf{W}^\star \mathbf{X} \;\preceq\; \mathbf{X}^\top \boldsymbol{\Lambda} \mathbf{W}^\star \mathbf{X},
\]
so inversion reverses the order and
\[
\tr\!\left((\mathbf{X}^\top \boldsymbol{\Lambda} \mathbf{W}^\star \mathbf{X})^{-1}\right)
\;\le\; \frac{1}{\min_j w_j\sigma_j^2}\; \tr\!\left((\mathbf{X}^\top \mathbf{W}^\star \mathbf{X})^{-1}\right)
= \frac{1}{\min_j w_j\sigma_j^2}\; \tr\!\left(\Var(\hat{\boldsymbol{\theta}}^\star)\right).
\]
Combining the displays yields the claim.
\end{proof}

Next, we establish a concentration inequality that allows us to bound the weight updates produced by the algorithm at each iteration.

\begin{lemma}
\label{lem:latent-hoeffding}
Consider the latent variable model 
$y = \mathbf{x}^\top \boldsymbol{\theta}_{\mathrm{gt}} + \varepsilon$,
where $\varepsilon$ denotes bounded noise satisfying $|\varepsilon| \le R_\varepsilon$, 
and $\mathbf{x}$ satisfies $\|\mathbf{x}\|_2 \le R_x$.
Let $\mathcal{B}=\{(\mathbf{x}_1,y_1),\ldots,(\mathbf{x}_n,y_n)\}$ be i.i.d.\ samples, and define the squared loss
\[
\ell(\boldsymbol{\theta},(\mathbf{x},y))
=(\mathbf{x}^\top\boldsymbol{\theta}-y)^2
=(\mathbf{x}^\top(\boldsymbol{\theta}-\boldsymbol{\theta}_{\mathrm{gt}})-\varepsilon)^2.
\]
Then for any fixed $\boldsymbol{\theta}\in\mathbb{R}^d$ and $\delta\in(0,1)$,
with probability at least $1-\delta$,
\[
\Bigg|
\frac{1}{n}\sum_{j=1}^n \ell(\boldsymbol{\theta},(\mathbf{x}_j,y_j))
-\Big(
(\boldsymbol{\theta}-\boldsymbol{\theta}_{\mathrm{gt}})^\top
\Sigma_x(\boldsymbol{\theta}-\boldsymbol{\theta}_{\mathrm{gt}})
+\sigma_\varepsilon^2
\Big)
\Bigg|
\le
(R_x\|\boldsymbol{\theta}-\boldsymbol{\theta}_{\mathrm{gt}}\|_2+R_\varepsilon)^2
\sqrt{\frac{\log(2/\delta)}{2n}},
\]
where $\Sigma_x=\mathbb{E}[\mathbf{x}\mathbf{x}^\top]$ and $\sigma_\varepsilon^2=\mathbb{E}[\varepsilon^2]$.
\end{lemma}

\begin{proof}
For any $(\mathbf{x},y)$, since $\|\mathbf{x}\|_2\le R_x$ and $|\varepsilon|\le R_\varepsilon$, we have
\[
|\mathbf{x}^\top(\boldsymbol{\theta}-\boldsymbol{\theta}_{\mathrm{gt}})-\varepsilon|
\le R_x\|\boldsymbol{\theta}-\boldsymbol{\theta}_{\mathrm{gt}}\|_2+R_\varepsilon.
\]
Hence,
\[
0\le \ell(\boldsymbol{\theta},(\mathbf{x},y))
\le (R_x\|\boldsymbol{\theta}-\boldsymbol{\theta}_{\mathrm{gt}}\|_2+R_\varepsilon)^2.
\]
By Hoeffding’s inequality, for i.i.d.\ random variables bounded in $[0,\,U]$ with $U=(R_x\|\boldsymbol{\theta}-\boldsymbol{\theta}_{\mathrm{gt}}\|_2+R_\varepsilon)^2$, we have for any $\epsilon>0$,
\[
\Pr\!\left(
\Bigg|\frac{1}{n}\sum_{j=1}^n \ell(\boldsymbol{\theta},(\mathbf{x}_j,y_j))
-\mathbb{E}[\ell(\boldsymbol{\theta},(\mathbf{x},y))]\Bigg|
\ge \epsilon
\right)
\le 2\exp\!\left(-\frac{2n\epsilon^2}{U^2}\right).
\]
Setting the right-hand side equal to $\delta$ and solving for $\epsilon$ yields, with probability at least $1-\delta$,
\[
\Bigg|\frac{1}{n}\sum_{j=1}^n \ell(\boldsymbol{\theta},(\mathbf{x}_j,y_j))
-\mathbb{E}[\ell(\boldsymbol{\theta},(\mathbf{x},y))]\Bigg|
\le
U\sqrt{\frac{\log(2/\delta)}{2n}}.
\]

Under the latent model $y=\mathbf{x}^\top\boldsymbol{\theta}_{\mathrm{gt}}+\varepsilon$ with $\mathbb{E}[\varepsilon]=0$ and $\mathbb{E}[\varepsilon^2]=\sigma_\varepsilon^2$, we have
\[
\mathbb{E}[\ell(\boldsymbol{\theta},(\mathbf{x},y))]
=(\boldsymbol{\theta}-\boldsymbol{\theta}_{\mathrm{gt}})^\top
\Sigma_x(\boldsymbol{\theta}-\boldsymbol{\theta}_{\mathrm{gt}})
+\sigma_\varepsilon^2.
\]
Substituting this into the bound gives the claimed two-sided inequality.
\end{proof}

Next, we introduce some useful notations. Let $\hat{\boldsymbol{\theta}}_m$ denote the weighted ERM solution corresponding to the updated weights at time step $mT_0 - 1$. For instance, $\hat{\boldsymbol{\theta}}_0$ represents the standard (unweighted) ERM solution. In addition, let $\boldsymbol{\theta}^\star_{\mathcal{S}}$ denote the optimal ERM solution obtained from \Cref{thm:aitken} on the set $\mathcal{S}$. Finally, let $\mathcal{S}_{\text{train}}$ be the subset of data initially sampled at random according to the initial ratio $\rho$ and used for training.

\begin{lemma}\label{lem:one-step-contraction}
Set $\gamma = 1$. 
Assume bounded data, $\|\mathbf{x}\|_2 \le R_x$ and $|\varepsilon| \le R_\varepsilon$, and a hypothesis class with finite diameter
\(
D \coloneqq \sup_{\boldsymbol{\theta} \in \mathcal H} \|\boldsymbol{\theta} - \boldsymbol{\theta}_{\mathrm{gt}}\| < \infty.
\)
Let $\delta \in (0,1)$ and define $\delta' = \delta T_0 / T$.  
Assume $B'_{\min} = \min_i M|\mathcal{S}_i|$ in \Cref{alg:one-shot-fgls} and $\sigma_{\min} = \min_i \sigma_i$.  
Suppose
\[
B'_{\min} \ge 8\!\left(\frac{R_x^2 D^2 + R_\varepsilon^2}{\sigma_{\min}^2}\right)^{\!2}
\log\!\frac{2K}{\delta'},
\]
and that there exists $\Delta_{\mathrm{op}} > 0$ such that
\[
\|\boldsymbol{\theta}_{mT_0} - \hat{\boldsymbol{\theta}}_{m-1}\| \le \Delta_{\mathrm{op}}
\quad \text{for all } 1 \le m \le \tfrac{T}{T_0}.
\]

Define
\[
\Sigma_{\max} = \max_i 
\frac{\big\| \mathbb{E}_{(\mathbf{x}, y) \sim \mathcal{D}_i} [\mathbf{x} \mathbf{x}^\top] \big\|_2}{\sigma_i^2},
\qquad
C_1 \coloneqq \Sigma_{\max}
      + 4\,\frac{R_x^2}{\sigma_{\min}^2}\sqrt{\tfrac{\log(2K/\delta')}{2B'_{\min}}},
\]
\[
C_2 \coloneqq 8\,\frac{R_\varepsilon^2}{\sigma_{\min}^2}\sqrt{\tfrac{\log(2K/\delta')}{2B'_{\min}}},
\qquad
\Delta_{\mathcal S}^\star \coloneqq 
\mathbb{E}\!\left[\|\boldsymbol{\theta}^\star_{\mathcal S} - \boldsymbol{\theta}_{\mathrm{gt}}\|^2\right].
\]

If $4C_1 \Delta_{\mathcal S_{\text{train}}}^\star < 1$, then with probability at least $1 - \delta$, for any 
$1 \le m \le \tfrac{T}{T_0}$,
\[
\mathbb{E}\!\left[\|\hat{\boldsymbol{\theta}}_m - \boldsymbol{\theta}_{\mathrm{gt}}\|^2\right]
\le
\big(4C_1 \Delta_{\mathcal S_{\text{train}}}^\star\big)^m
\mathbb{E}\!\left[\|\hat{\boldsymbol{\theta}}_0 - \boldsymbol{\theta}_{\mathrm{gt}}\|^2\right]
+
\frac{\Delta_{\mathcal S_{\text{train}}}^\star\!\left(1 + 4C_1\Delta_{\mathrm{op}}^2 + C_2\right)}
     {1 - 4C_1 \Delta_{\mathcal S_{\text{train}}}^\star}.
\]
Consequently,
\[
\frac{\mathbb{E}\!\left[\|\hat{\boldsymbol{\theta}}_m - \boldsymbol{\theta}_{\mathrm{gt}}\|^2\right]}
     {\mathbb{E}\!\left[\|\boldsymbol{\theta}^\star_{\mathcal S_{\text{train}}} - \boldsymbol{\theta}_{\mathrm{gt}}\|^2\right]}
\le
4C_1\!\left(4C_1 \Delta_{\mathcal S_{\text{train}}}^\star\right)^{m-1} D^2
+
\frac{1 + 4C_1\Delta_{\mathrm{op}}^2 + C_2}
     {1 - 4C_1 \Delta_{\mathcal S_{\text{train}}}^\star}.
\]
\end{lemma}

\begin{proof}
At iteration $t = mT_0 - 1$, the domain-$i$ weight is
\[
  w_i^{(mT_0)} = \left(\frac{1}{|\mathcal{B}'_i|} \sum_{\mathbf{z} \in \mathcal{B}'_i} \ell(\boldsymbol{\theta}_{mT_0}, \mathbf{z})\right)^{-1}.
\]
By \Cref{lem:latent-hoeffding} and a union bound, with probability at least $1 - \delta'$, for all $i \in [K]$,
\begin{equation}\label{eq:wbounds}
  \frac{1}{1 + \Sigma_{\max}\|\boldsymbol{\theta}_{mT_0} - \boldsymbol{\theta}_{\mathrm{gt}}\|^2 + \zeta_i}
  \;\le\;
  \sigma_i^2 w_i^{(mT_0)}
  \;\le\;
  \frac{1}{1 - \zeta_i},
\end{equation}
where
\[
  \zeta_i \coloneqq \frac{R_{mT_0}}{\sigma_i^2}\sqrt{\frac{\log(2K/\delta')}{2|\mathcal{B}'_i|}},
  \qquad
  R_{mT_0} \coloneqq 2\!\left(R_x^2\|\boldsymbol{\theta}_{mT_0} - \boldsymbol{\theta}_{\mathrm{gt}}\|^2 + R_\varepsilon^2\right).
\]

Combining \eqref{eq:wbounds} with \Cref{lem:estimator_var_bound} yields
\begin{equation}\label{eq:ratio-initial}
  \frac{\mathbb{E}\!\left[\|\hat{\boldsymbol{\theta}}_m - \boldsymbol{\theta}_{\mathrm{gt}}\|^2\right]}
       {\mathbb{E}\!\left[\|\boldsymbol{\theta}^\star_{\mathcal S_{\text{train}}} - \boldsymbol{\theta}_{\mathrm{gt}}\|^2\right]}
  \le
  1 + \frac{\Sigma_{\max}\|\boldsymbol{\theta}_{mT_0} - \boldsymbol{\theta}_{\mathrm{gt}}\|^2 + 2\zeta_{\max}}{1 - \zeta_{\max}},
\end{equation}
where
\[
  \zeta_{\max} \coloneqq \frac{R_{mT_0}}{\sigma_{\min}^2}\sqrt{\frac{\log(2K/\delta')}{2B'_{\min}}}.
\]

Under the batch-size and bounded-diameter assumptions, $\zeta_{\max} \le \tfrac{1}{2}$, so
\begin{align}
\frac{\mathbb{E}\!\left[\|\hat{\boldsymbol{\theta}}_m - \boldsymbol{\theta}_{\mathrm{gt}}\|^2\right]}
     {\mathbb{E}\!\left[\|\boldsymbol{\theta}^\star_{\mathcal S_{\text{train}}} - \boldsymbol{\theta}_{\mathrm{gt}}\|^2\right]}
&\le
1 + 2\!\left(\Sigma_{\max}\|\boldsymbol{\theta}_{mT_0} - \boldsymbol{\theta}_{\mathrm{gt}}\|^2 + 2\zeta_{\max}\right)\notag\\
&=
1 + 2\|\boldsymbol{\theta}_{mT_0} - \boldsymbol{\theta}_{\mathrm{gt}}\|^2
  \!\left(\Sigma_{\max} + 4\,\tfrac{R_x^2}{\sigma_{\min}^2}\sqrt{\tfrac{\log(2K/\delta')}{2B'_{\min}}}\right)
+ 8\,\tfrac{R_\varepsilon^2}{\sigma_{\min}^2}\sqrt{\tfrac{\log(2K/\delta')}{2B'_{\min}}}.
\label{eq:simplified}
\end{align}

By using
\begin{equation}\label{eq:split}
  \|\boldsymbol{\theta}_{mT_0} - \boldsymbol{\theta}_{\mathrm{gt}}\|^2
  \le
  2\|\boldsymbol{\theta}_{mT_0} - \hat{\boldsymbol{\theta}}_{m-1}\|^2
  + 2\|\hat{\boldsymbol{\theta}}_{m-1} - \boldsymbol{\theta}_{\mathrm{gt}}\|^2,
\end{equation}
and $\|\boldsymbol{\theta}_{mT_0} - \hat{\boldsymbol{\theta}}_{m-1}\| \le \Delta_{\mathrm{op}}$, we get
\[
  \|\boldsymbol{\theta}_{mT_0} - \boldsymbol{\theta}_{\mathrm{gt}}\|^2 \le 2\Delta_{\mathrm{op}}^2 + 2\|\hat{\boldsymbol{\theta}}_{m-1} - \boldsymbol{\theta}_{\mathrm{gt}}\|^2.
\]

Substituting into \eqref{eq:simplified}, multiplying by $\Delta_{\mathcal S_{\text{train}}}^\star$, and collecting constants gives
\begin{equation}\label{eq:recursion}
  \mathbb{E}\!\left[\|\hat{\boldsymbol{\theta}}_m - \boldsymbol{\theta}_{\mathrm{gt}}\|^2\right]
  \le
  \Delta_{\mathcal S_{\text{train}}}^\star
  \!\left(1 + 4C_1\Delta_{\mathrm{op}}^2
        + 4C_1\,\mathbb{E}\!\left[\|\hat{\boldsymbol{\theta}}_{m-1} - \boldsymbol{\theta}_{\mathrm{gt}}\|^2\right]
        + C_2\right).
\end{equation}
If $4C_1 \Delta_{\mathcal S_{\text{train}}}^\star < 1$, unrolling \eqref{eq:recursion} and applying a union bound yields, with probability at least $1 - \delta$,
\[
  \mathbb{E}\!\left[\|\hat{\boldsymbol{\theta}}_m - \boldsymbol{\theta}_{\mathrm{gt}}\|^2\right]
  \le
  \big(4C_1\Delta_{\mathcal S_{\text{train}}}^\star\big)^m
  \mathbb{E}\!\left[\|\hat{\boldsymbol{\theta}}_0 - \boldsymbol{\theta}_{\mathrm{gt}}\|^2\right]
  +
  \frac{\Delta_{\mathcal S_{\text{train}}}^\star (1 + 4C_1\Delta_{\mathrm{op}}^2 + C_2)}
       {1 - 4C_1\Delta_{\mathcal S_{\text{train}}}^\star},
\]
completing the proof by using the diameter assumption.
\end{proof}
We now present the main convergence result.

\begin{theorem}[Formal]\label{thm:formal}
Consider the assumptions in \Cref{lem:one-step-contraction}. 
Suppose $\lim_{T_0,\, B \to \infty} \Delta_{\mathrm{op}} = 0$ (as in smooth and convex SGD), and that the ratio $T' = T / T_0$ is fixed. 
Let $\rho = 1 - \tfrac{1}{\sqrt{|\mathcal{S}|}}$, where $|\mathcal{S}|$ denotes the total number of data points. 
Assume there exists a finite constant $C_3 < \infty$ such that
\[
C_3 = \sup
\frac{\mathbb{E}\!\left[\|\hat{\boldsymbol{\theta}}_{T'} - \boldsymbol{\theta}_{\mathrm{gt}}\|^2\right]}
     {\mathbb{E}\!\left[\|\boldsymbol{\theta}^{\star}_{\mathcal{S}_{\mathrm{train}}} - \boldsymbol{\theta}_{\mathrm{gt}}\|^2\right]}.
\]
Then,
\[
\lim_{|\mathcal{S}| \to \infty} 
\frac{\mathbb{E}\!\left[\|\hat{\boldsymbol{\theta}}_{T'} - \boldsymbol{\theta}_{\mathrm{gt}}\|^2\right]}
     {\mathbb{E}\!\left[\|\boldsymbol{\theta}^{\star}_{\mathcal{S}_{\mathrm{train}}} - \boldsymbol{\theta}_{\mathrm{gt}}\|^2\right]} 
= 1.
\]
\end{theorem}

\begin{proof}
Set $\delta = 1 / |\mathcal{S}|$. From \Cref{lem:one-step-contraction}, we have
\[
\frac{\mathbb{E}\!\left[\|\hat{\boldsymbol{\theta}}_{T'} - \boldsymbol{\theta}_{\mathrm{gt}}\|^2\right]}
     {\mathbb{E}\!\left[\|\boldsymbol{\theta}^{\star}_{\mathcal{S}_{\mathrm{train}}} - \boldsymbol{\theta}_{\mathrm{gt}}\|^2\right]}
\le
4C_1\!\left(4C_1 \Delta_{\mathcal S_{\text{train}}}^\star\right)^{T'-1} D^2
+
\frac{1 + 4C_1\Delta_{\mathrm{op}}^2 + C_2}{1 - 4C_1 \Delta_{\mathcal S_{\text{train}}}^\star}
+ \delta C_3.
\]
Taking the limit as $|\mathcal{S}| \to \infty$, and using 
\(\lim_{|\mathcal{S}| \to \infty} \Delta^{\star}_{\mathcal{S}_{\mathrm{train}}} = 0\),
\(\lim_{|\mathcal{S}| \to \infty} C_2 = 0\), and 
\(\lim_{|\mathcal{S}| \to \infty} C_1 = \Sigma_{\max}\)
(since \(\lim_{|\mathcal{S}| \to \infty} B'_{\min} = \infty\)),
we obtain
\begin{equation}
\lim_{|\mathcal{S}| \to \infty}
\frac{\mathbb{E}\!\left[\|\hat{\boldsymbol{\theta}}_{T'} - \boldsymbol{\theta}_{\mathrm{gt}}\|^2\right]}
     {\mathbb{E}\!\left[\|\boldsymbol{\theta}^{\star}_{\mathcal{S}_{\mathrm{train}}} - \boldsymbol{\theta}_{\mathrm{gt}}\|^2\right]}
\le
1 + 4\,\Sigma_{\max}\,\Delta_{\mathrm{op}}^2.
\label{eq:asymptotic-bound}
\end{equation}
Moreover, by asymptotic normality,
\[
\lim_{|\mathcal{S}| \to \infty}
\frac{\mathbb{E}\!\left[\|\boldsymbol{\theta}^{\star}_{\mathcal{S}_{\mathrm{train}}} - \boldsymbol{\theta}_{\mathrm{gt}}\|^2\right]}
     {\mathbb{E}\!\left[\|\boldsymbol{\theta}^{\star}_{\mathcal{S}} - \boldsymbol{\theta}_{\mathrm{gt}}\|^2\right]}
= 
\lim_{|\mathcal{S}| \to \infty} \frac{1}{\rho} = 1.
\]
Combining these results gives
\[
1 \le
\lim_{|\mathcal{S}| \to \infty}
\frac{\mathbb{E}\!\left[\|\hat{\boldsymbol{\theta}}_{T'} - \boldsymbol{\theta}_{\mathrm{gt}}\|^2\right]}
     {\mathbb{E}\!\left[\|\boldsymbol{\theta}^{\star}_{\mathcal{S}_{\mathrm{train}}} - \boldsymbol{\theta}_{\mathrm{gt}}\|^2\right]}
\le
1 + 4\,\Sigma_{\max}\,\Delta_{\mathrm{op}}^2.
\]
Finally, since $\Delta_{\mathrm{op}} \to 0$ as $T_0, B \to \infty$, the result follows.
\end{proof}

\subsection{Proofs of \Cref{thm:generalization-bound}}
\begin{theorem}[Formal]
\label{thm:generalization-bound-formal}
Assume the loss function $\ell(\boldsymbol{\theta}, \mathbf{z})$ is bounded by $L_{\max}$ for all $\boldsymbol{\theta}$ and $\mathbf{z}$, and that 
$\Var_i(\boldsymbol{\theta}) \le \Var_{\max}$ for all domains $i \in [K]$.
Let $\{\boldsymbol{\theta}_{mT_0}\}_{m=1}^{T/T_0}$ denote the iterates produced by \Cref{alg:one-shot-fgls}.
Then, for any $\delta \in (0,1)$, with probability at least $1 - \delta$ over the random draw of the 
validation sets $\mathcal{V} = \{\mathcal{V}_1, \ldots, \mathcal{V}_K\}$, the following holds 
simultaneously for all iterates $\boldsymbol{\theta}_{mT_0}$, provided that for every domain $i$,
\[
|\mathcal{V}_i| \;\ge\; \frac{L_{\max}^2\,\ln\!\bigl(KT/(\delta T_0)\bigr)}{18\,\Var_{\max}}.
\]
In that case,
\begin{align}
\bigl(\mathcal{L}_\pi(\boldsymbol{\theta}_{mT_0}) - \hat{\mathcal{L}}_{\mathcal{V}, \pi, w}(\boldsymbol{\theta}_{mT_0})\bigr)^2
&\leq
2\!\left(\sum_{i=1}^{K} \pi_i (1 - w_i)\, \mathcal{L}_i(\boldsymbol{\theta}_{mT_0})\right)^2 \nonumber\\
&\quad + 16K\, \ln\!\left(\frac{KT}{\delta T_0}\right)
\sum_{i=1}^{K} \pi_i^2 w_i^2\, 
\frac{\Var_i(\boldsymbol{\theta}_{mT_0})}{|\mathcal{V}_i|}.
\label{eq:formal-bound}
\end{align}
\end{theorem}

\begin{proof}
The proof proceeds by decomposing the deviation between the population and empirical losses. 
By definition,
\begin{align*}
\mathcal{L}_{\pi}(\boldsymbol{\theta}) - \hat{\mathcal{L}}_{\mathcal{V}, \pi, w}(\boldsymbol{\theta})
&= \sum_{i=1}^{K} \pi_i \!\left(w_i\, \mathcal{L}_i(\boldsymbol{\theta}) - w_i\, \hat{\mathcal{L}}_{\mathcal{V}_i}(\boldsymbol{\theta})\right)
+ \sum_{i=1}^{K} \pi_i (1 - w_i)\, \mathcal{L}_i(\boldsymbol{\theta}).
\end{align*}

Applying the AM--GM inequality, we have
\[
\bigl(\mathcal{L}_{\pi}(\boldsymbol{\theta}) - \hat{\mathcal{L}}_{\mathcal{V}, \pi, w}(\boldsymbol{\theta})\bigr)^2
\leq
2\!\left(\sum_{i=1}^{K} \pi_i (1 - w_i)\, \mathcal{L}_i(\boldsymbol{\theta})\right)^2
+ 2\!\left(\sum_{i=1}^{K} \pi_i w_i\, \bigl(\mathcal{L}_i(\boldsymbol{\theta}) - \hat{\mathcal{L}}_{\mathcal{V}_i}(\boldsymbol{\theta})\bigr)\right)^2.
\]
Next, applying the Cauchy--Schwarz inequality yields
\[
\bigl(\mathcal{L}_{\pi}(\boldsymbol{\theta}) - \hat{\mathcal{L}}_{\mathcal{V}, \pi, w}(\boldsymbol{\theta})\bigr)^2
\leq
2\!\left(\sum_{i=1}^{K} \pi_i (1 - w_i)\, \mathcal{L}_i(\boldsymbol{\theta})\right)^2
+ 2K \sum_{i=1}^{K} \pi_i^2 w_i^2 \bigl(\mathcal{L}_i(\boldsymbol{\theta}) - \hat{\mathcal{L}}_{\mathcal{V}_i}(\boldsymbol{\theta})\bigr)^2.
\]

Under the bounded loss assumption ($\ell(\boldsymbol{\theta}, \mathbf{z}) \le L_{\max}$) and bounded variance assumption ($\Var_i(\boldsymbol{\theta}) \le \Var_{\max}$), Bennett’s inequality implies that, for each domain $i$ and any $\delta_i > 0$, with probability at least $1 - \delta_i$,
\[
\mathcal{L}_i(\boldsymbol{\theta}) - \hat{\mathcal{L}}_{\mathcal{V}_i}(\boldsymbol{\theta})
\le
\sqrt{\frac{2\, \Var_i(\boldsymbol{\theta})\, \ln(1/\delta_i)}{|\mathcal{V}_i|}}
+ \frac{L_{\max}\, \ln(1/\delta_i)}{3|\mathcal{V}_i|}.
\]
Setting $\delta_i = \delta/(KT/T_0)$ and using $a+b \le 2\max\{a,b\}$, we obtain
\[
\mathcal{L}_i(\boldsymbol{\theta}) - \hat{\mathcal{L}}_{\mathcal{V}_i}(\boldsymbol{\theta})
\le
2\sqrt{\frac{2\, \Var_i(\boldsymbol{\theta})\, \ln(KT/(\delta T_0))}{|\mathcal{V}_i|}}.
\]
By $|\mathcal{V}_i| \ge \frac{L_{\max}^2\ln(KT/(\delta T_0))}{18\,\Var_{\max}}$, the second term of the Bennett bound is dominated by the variance term, validating the simplification above.

Taking a union bound over all $K$ domains and over the algorithm’s iterates $\{\boldsymbol{\theta}_{mT_0}\}_{m=1}^{T/T_0}$ ensures that, with probability at least $1 - \delta$, the above inequality holds uniformly for all $\boldsymbol{\theta}_{mT_0}$. 
Substituting this bound into the previous inequality gives
\[
\bigl(\mathcal{L}_{\pi}(\boldsymbol{\theta}_{mT_0}) - \hat{\mathcal{L}}_{\mathcal{V}, \pi, w}(\boldsymbol{\theta}_{mT_0})\bigr)^2
\le
2\!\left(\sum_{i=1}^{K} \pi_i (1 - w_i)\, \mathcal{L}_i(\boldsymbol{\theta}_{mT_0})\right)^2
+ 16K\, \ln\!\left(\frac{KT}{\delta T_0}\right)
\sum_{i=1}^{K} \pi_i^2 w_i^2 \frac{\Var_i(\boldsymbol{\theta}_{mT_0})}{|\mathcal{V}_i|}.
\]
This completes the proof.
\end{proof}
\clearpage
\section{ERMA Derivation}
In this section, we present the detailed derivation of the ERMA formulation.  
Recall that the update rule for mirror descent with the Kullback--Leibler (KL) divergence as its Bregman distance is given by
\[
\mathbf{w}_i^{(t+1)} \gets \mathbf{w}_i^{(t)} 
\exp\!\left(-\eta\, \nabla_{\mathbf{w}_i} f(\mathbf{w}^{(t)}) \right),
\]
where $\eta$ denotes the learning rate.  

We define the upper-bound objective function as
\[
f(\mathbf{w}) = 2\!\left(\sum_{i=1}^K \pi_i (1 - \mathbf{w}_i)\, 
       \mathcal{L}_i(\theta)\right)^{\!2}
       + C \sum_{i=1}^K 
       \frac{\pi_i^2 \mathbf{w}_i^2}{|\mathcal{V}_i|}\, 
       \Var_i(\theta).
\]

Taking the gradient of the upper bound $f(\mathbf{w})$ in \Cref{thm:generalization-bound}, we obtain
\[
\nabla_{\mathbf{w}_i} f(\mathbf{w}) 
= -4\pi_i \!\left(\sum_{j=1}^K \pi_j (1 - \mathbf{w}_j)\, 
       \mathcal{L}_j(\theta)\right) \mathcal{L}_i(\theta)
       + 2C\,\frac{\pi_i \mathbf{w}_i}{|\mathcal{V}|}\, 
       \Var_i(\theta),
\]
where we use $|\mathcal{V}_i| = \pi_i |\mathcal{V}|$.  

Substituting this gradient into the mirror descent update yields
\[
\mathbf{w}_i^{(t+1)} \propto 
\mathbf{w}_i^{(t)} \exp\!\left(
      \gamma_1\, \pi_i G(t)\, \mathcal{L}_i(\theta_t) 
      - \gamma_2\, \pi_i \mathbf{w}_i^{(t)}\,
        \Var_i(\theta_t)
    \right),
\]
where the constants are defined as $\gamma_1 = 4\eta \pi$ and $\gamma_2 = \eta \tfrac{2C}{|\mathcal{V}|}$.

\clearpage
\section{More Experiments}
\paragraph{Linear Regression} The first set of experiments investigates linear regression to isolate the effect of each method. Specifically, we consider two additional setups: (i) $(C_1, C_2) = (100, 1)$ with $(\sigma_1^2, \sigma_2^2) = (1, 1)$, and (ii) $(C_1, C_2) = (1, 1)$ with $(\sigma_1^2, \sigma_2^2) = (1, 20)$; see \Cref{fig:linear-regression-2}. As shown, both VA and One-shot FGLS improve upon the standard (vanilla) training baseline.

\paragraph{Logistic Regression} We also report the test accuracy performance of different models in the logistic regression example shown in \Cref{fig:logistic-regression-2}. As illustrated in the figures, both ERMA and VA improve the accuracy metric, alongside the improvement observed in terms of cosine distance to the ground truth.

\begin{figure*}[!t]
    \centering
    \includegraphics[width=1\linewidth]{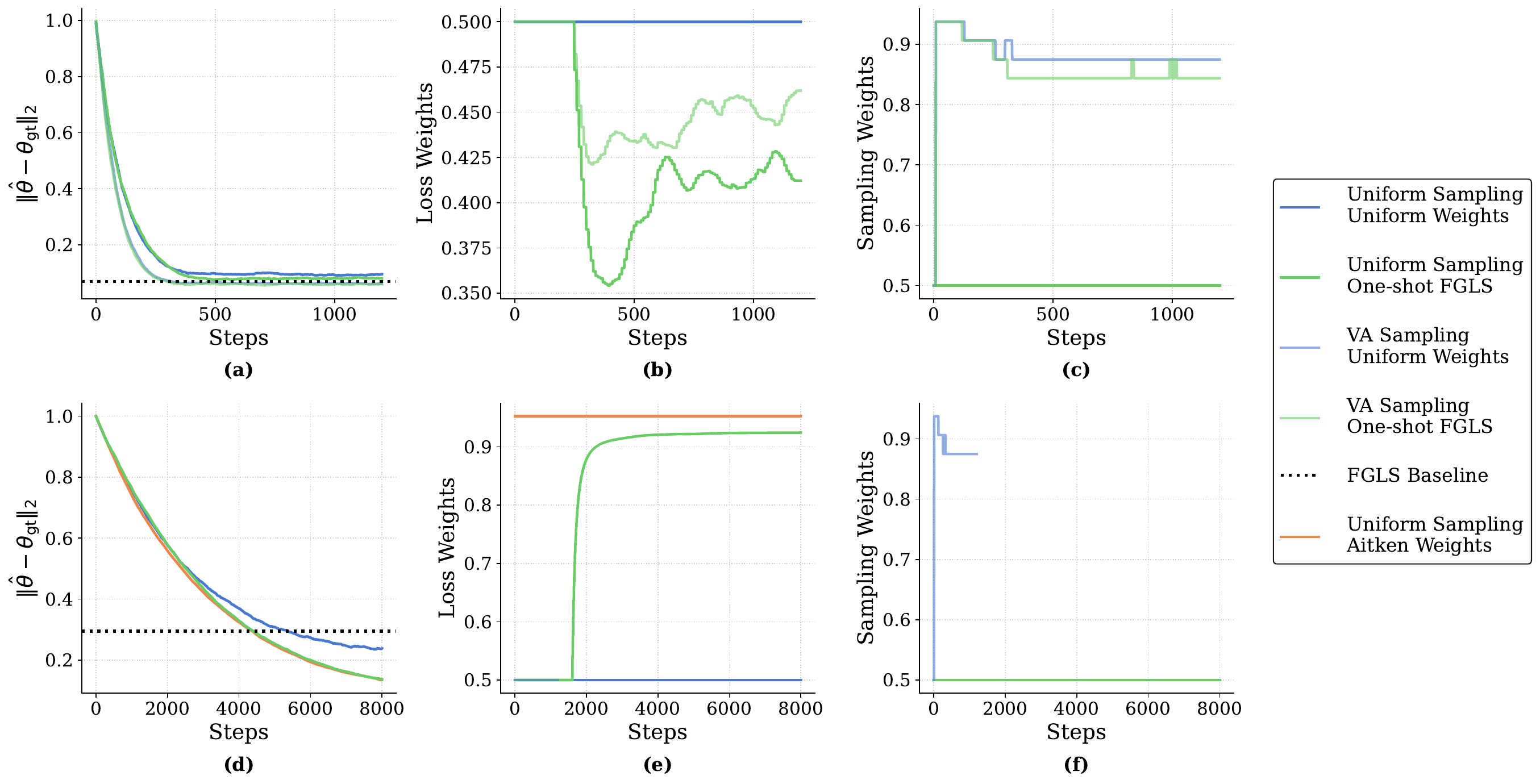}
    \caption{Performance of different methods in the linear regression example. 
        Figures \textbf{a to c} correspond to $(C_1, C_2) = (100, 1)$ and $(\sigma_1^2, \sigma_2^2) = (1,1)$, while Figures \textbf{d to f} correspond to $(C_1, C_2) = (1, 1)$ and $(\sigma_1^2, \sigma_2^2) = (1,20)$. 
        \textbf{a, d:} Distance between the estimated parameter and the ground-truth $\theta_{\mathrm{gt}}$ for each method. 
        \textbf{b, e:} Evolution of loss weights for domain one during training. 
        \textbf{c, f:} Evolution of sampling weights for domain one during training.
    }
    \label{fig:linear-regression-2}
\end{figure*}

\begin{figure*}[!t]
    \centering
    \includegraphics[width=1\linewidth]{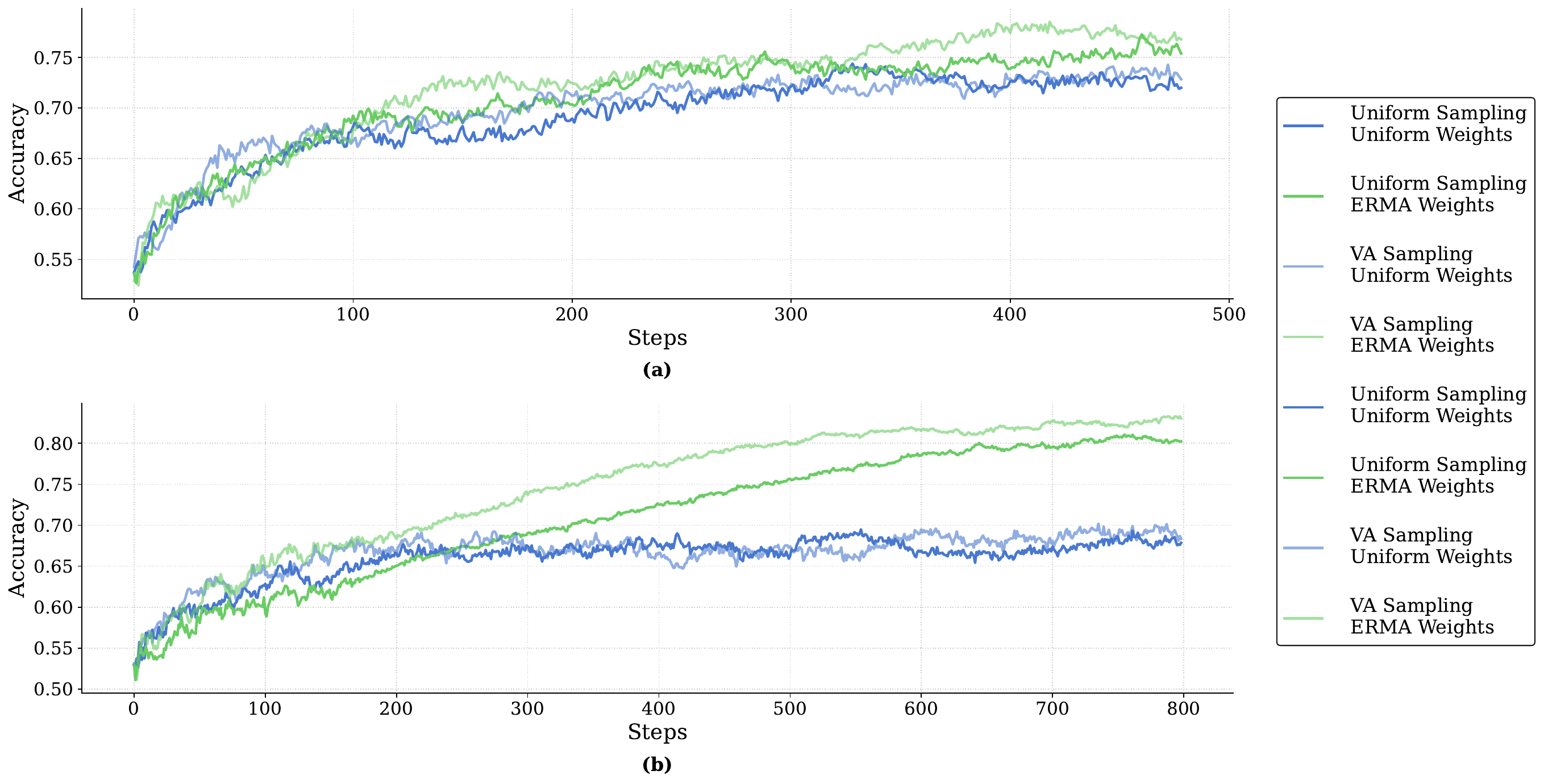}
    \caption{Performance of different methods in the logistic regression example under accuracy.
Figure \textbf{(a)} corresponds to \( (C_1, C_2) = (100, 100) \), while Figure \textbf{(b)} corresponds to \( (C_1, C_2) = (10, 100) \). 
    }
    \label{fig:logistic-regression-2}
\end{figure*}
\section{Using a Single Weight}
In this section, we evaluate the effect of combining VA and ERMA weights into a single set of sampling weights. To achieve this, we multiply the corresponding ERMA and VA weights for each domain and then normalize the resulting values. (We use uniform loss weights for this new algorithm.) \Cref{fig:ablation} shows that this combined approach yields suboptimal results, highlighting the importance of maintaining separate loss and sampling weights.
\begin{figure*}[!t]
    \centering
    \includegraphics[width=1\linewidth]{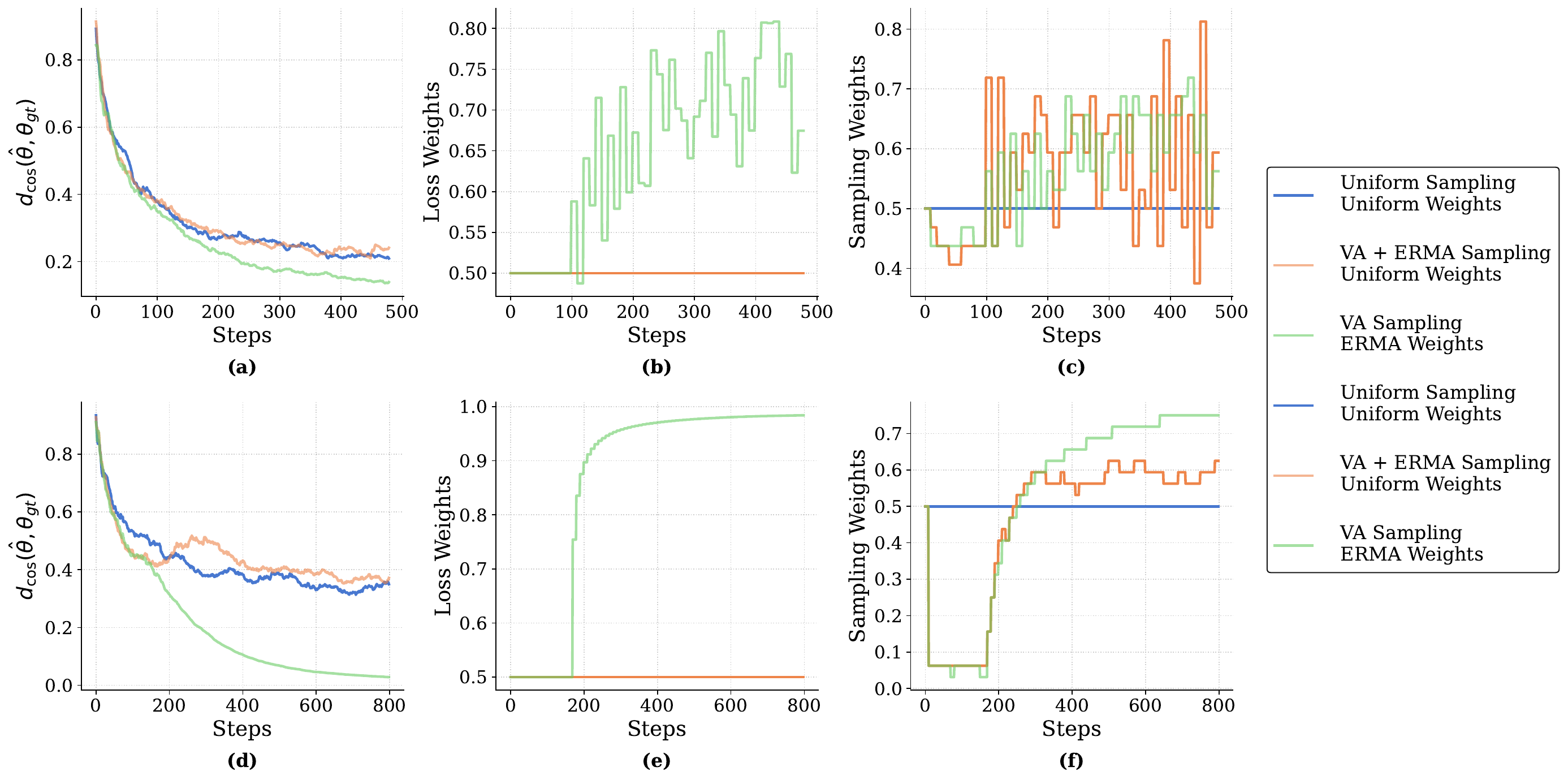}
    \caption{Comparison of vanilla training, training with ERMA loss and VA sampling weights, and a combined approach that merges ERMA and VA into a single set of sampling weights.
    }
    \label{fig:ablation}
\end{figure*}

\clearpage

\end{document}